\documentclass{article}

\usepackage[preprint]{neurips_2026}
\usepackage{amsmath}

\usepackage{amsthm} 

\usepackage{algorithm}
\usepackage{algorithmic}

\usepackage{graphicx}

\usepackage{subcaption}

\newcommand{\best}[1]{\textbf{#1}}
\newcommand{\sbest}[1]{\underline{#1}}

\providecommand{\sbest}[1]{\underline{#1}}

\newtheorem{theorem}{Theorem}

\newtheorem{corollary}[theorem]{Corollary}
\newtheorem{definition}{Definition}

\usepackage[utf8]{inputenc} 
\usepackage[T1]{fontenc}    
\usepackage{hyperref}       
\usepackage{url}            
\usepackage{booktabs}       
\usepackage{amsfonts}       
\usepackage{nicefrac}       
\usepackage{microtype}      
\usepackage{xcolor}         

\title{NeuraLSP: A Neural Spectral Preconditioner for Accelerating PDE Solvers}

\author{%
  Alexander Benanti\thanks{Lead author.} \\
  Department of Applied Mathematics and Statistics\\
  Stony Brook University\\
  Stony Brook, NY 11794 \\
  \texttt{alexander.benanti@stonybrook.edu} \\
   \And
   Xi Han\thanks{Co-first author} \\
  Department of Computer Science\\
  Stony Brook University\\
  Stony Brook, NY 11794 \\
    \texttt{xihan1@cs.stonybrook.edu} \\
  \AND
  Hong Qin\footnotemark[3] \\
  Department of Computer Science \\
  Stony Brook University\\
  Stony Brook, NY 11794 \\
  \texttt{qin@cs.stonybrook.edu} \\
}

\begin{document}

\maketitle

\begin{abstract}
  Solving large-scale sparse linear systems originating from partial differential equations (PDEs) is a fundamental topic in high-performance scientific computing, where preconditioners are crucial. Multigrid methods are among the most effective preconditioners, yet their performance is dictated by the accurate construction of grid transfer operators. Current neural multigrid methods learn such operators with graph neural networks (GNNs), typically by extracting connectivity from discretized system matrices. While effective, these graph-based constructions suffer from rank inflation, resulting in unnecessarily large coarse spaces and slower convergence. To ameliorate, this paper advocates NeuraLSP, a new neural multigrid preconditioner that replaces graph aggregation with a fixed low-rank spectral representation derived from the left singular subspace of near-nullspace components. At the network design level, NeuraLSP is trained with a novel subspace loss function, which preserves the error modes most relevant to multigrid convergence while suppressing rank inflation. This paper's grand innovation hinges upon both theoretical guarantees and empirical robustness to rank inflation, affording up to a 53\% speedup over SOTA neural preconditioners across a variety of PDE families. Code is available at \href{https://github.com/alexbenanti/NeuraLSPv2}{https://github.com/alexbenanti/NeuraLSPv2}.
\end{abstract}

\section{Introduction and Motivation}

\textbf{Background and Key Challenges.} \quad
Partial differential equations (PDEs) are of fundamental significance to computational science and engineering, where they are ubiquitous in simulation, modeling, and scientific computing. After discretization, many PDE problems require solving large sparse linear systems of the form $\mathbf{A} \mathbf{u} = \mathbf{f}$. The cost of these linear solves is often a dominant computational bottleneck, particularly for high-resolution discretizations and repeated solves across parameterized problem families. 

Algebraic multigrid (AMG) methods are among the most widely used techniques for accelerating such systems. AMG combines relaxation, which damps high-frequency error components, with coarse-grid correction, which targets smooth error components associated with the near-nullspace of the system matrix \cite{Stuben2001, Falgout2006, RugeStuben87, Xu2017, Brandt1977, Briggs00, Trottenberg2001}. The effectiveness of AMG, therefore, depends critically on the construction of the prolongation operator, which transfers coarse-level corrections back to the fine grid. Classical AMG methods typically construct this operator using local strength-of-connection heuristics derived from the matrix graph. While effective in many settings, these heuristics can lead to dense coarse operators, especially when the near-nullspace structure is complex. This phenomenon increases operator complexity, setup cost, and solve time.

Recent neural approaches seek to improve multigrid preconditioning by learning prolongation operators or related solver components from data. Graph neural networks (GNNs) are a natural choice because sparse system matrices induce graph structures. However, many neural multigrid methods still rely on graph-based aggregation of the discretized operator, which can inherit the same coarse-space inflation issues as classical constructions. Moreover, existing learned solvers often lack explicit spectral control over the error modes that are most important for multigrid convergence. This motivates a central question: \textbf{\textit{Can a learned preconditioner construct a smaller coarse representation while preserving the near-nullspace information required for fast convergence?}}

\textbf{Motivation and Method Overview.} \quad 
We propose \textbf{NeuraLSP}, a neural multigrid preconditioning framework designed to address this question. Instead of learning prolongation solely from graph connectivity, NeuraLSP learns a fixed low-rank operator informed by the left singular subspace of the near-nullspace components of the system matrix. The resulting representation compresses the spectral information needed for coarse-grid correction while explicitly controlling the rank of the learned coarse space.

The key idea behind NeuraLSP is that \textbf{\textit{rank inflation}} arises when the learned or constructed coarse representation contains more degrees of freedom than are needed to capture the dominant smooth error modes. We use the term ``rank inflation'' to refer to the growth of the coarse representation beyond the dimension needed to capture the dominant smooth error modes, which leads to increased setup cost, excessive operator complexity, and long solve time. To avoid this, we introduce a neural low-rank preconditioner, described in Algorithm~\ref{alg:neuralsp}, and train it using a Nested Left-Singular-Subspace (NLSS) loss. This objective encourages the learned operator to recover the relevant singular subspace while maintaining a compact coarse representation. The full pipeline is illustrated in Fig.~\ref{fig:pipeline_detailed}.

NeuraLSP combines data-driven learning with explicit spectral structure. This design yields two complementary benefits. First, the learned representation is constrained to remain low rank, making the method robust to rank inflation. Second, the loss function directly targets the subspace associated with near-nullspace error components, aligning the learning objective with the mechanism that drives multigrid convergence. Our theoretical analysis establishes guarantees for the proposed objective, and our experiments show that these guarantees translate into improved efficiency and robustness across multiple PDE families.

\textbf{Major Contributions.} \quad
The salient contributions of this paper comprise: 
(1) \textbf{\textit{A neural low-rank multigrid preconditioner.}}
    We introduce \textbf{NeuraLSP}, a neural preconditioning framework that constructs fixed low-rank transfer operators using learned left singular subspaces. Unlike graph-aggregation-based prolongation learning, NeuraLSP explicitly controls the dimension of the coarse representation, improving robustness to rank inflation;
(2) \textbf{\textit{A left-singular-subspace training objective.}}
    We propose the \textbf{NLSS loss}, a nested subspace loss designed to recover the leading left singular subspace of the near-nullspace components of the system matrix. This objective provides a principled way to preserve the error modes most relevant to multigrid convergence while reducing coarse-space complexity; and
(3) \textbf{\textit{Theoretical and empirical validation.}}
    We prove theoretical guarantees for the proposed objective and evaluate NeuraLSP across diverse PDE benchmarks. Our experiments demonstrate improved solve-time efficiency, strong generalization across PDE variants, and robustness to coarse-space rank inflation compared with classical and neural preconditioning baselines.
To the best of our knowledge, NeuraLSP is \textbf{the first} neural multigrid preconditioning framework to use a learned left singular subspace to explicitly control rank inflation in the coarse representation.

\section{Related Work}

\textbf{Classical AMG and Near-Nullspace-Based Coarsening.} \quad
Multigrid methods accelerate the solution of large sparse linear systems by combining relaxation with coarse-grid correction. In algebraic multigrid (AMG), the central object is the prolongation operator $\mathbf{P}$, which maps coarse-level corrections to the fine level. Classical AMG methods construct $\mathbf{P}$ using local algebraic information, most commonly strength-of-connection measures derived from the matrix graph \cite{Stuben2001,Falgout2006,RugeStuben87,Xu2017}. These methods are effective for many unstructured problems, but their performance depends critically on whether the coarse space can represent algebraically smooth error modes. When the prolongation space is too large or poorly structured, the resulting hierarchy can suffer from high operator complexity, increased setup cost, and slower matrix-vector operations.

Smoothed-aggregation AMG (SA-AMG) addresses this issue by constructing tentative prolongation operators from near-nullspace candidate vectors and then smoothing them to improve approximation quality \cite{Vanek1996}. Beyond SA-AMG, \cite{Chow06} proposed using local singular value decompositions of smooth-error aggregates to construct multilevel interpolation. These methods demonstrate the value of spectral information for multigrid coarsening, but they rely on explicit SVD computations and are not designed as amortized learned preconditioners.

\textbf{Learning Spectral Subspaces with Neural Networks.} \quad
Several works have studied differentiable eigendecomposition and SVD-based layers or objectives for neural networks \cite{Wang2019,Levinson2020}. Other methods use spectral constraints or spectral normalization to stabilize learning and control operator behavior \cite{Miyato2018}. More recently, neural methods have been developed to learn spectral information of operators directly, including approaches for learning leading eigenfunctions, singular functions, or principal components \cite{Pfau2018,Shaham2018,Deng2022,Bao2020,Oftadeh2020,Gemp2021,Li2024,Ryu24}. Randomized SVD and related sketching methods further reduce the cost of approximating dominant spectral subspaces \cite{Halko2011}. However, these spectral learning methods are generally not designed to construct multigrid transfer operators, and they do not directly address the rank inflation that can occur in coarse-space construction for AMG.

\textbf{Learning Multigrid Transfer Operators.} \quad
Recent work has explored neural approaches for learning multigrid solvers and prolongation operators. Greenfeld et al. \cite{Greenfeld19} proposed learning a mapping from structured families of discretized PDEs to prolongation operators, demonstrating that neural networks can amortize multigrid design across parameterized problem classes. Luz et al. \cite{Luz20} extended this idea to algebraic multigrid by using graph neural networks to learn prolongation operators for sparse symmetric positive semidefinite systems. However, these approaches typically rely on graph-based representations of the system matrix. As a result, learned prolongation operators can still produce large or overparameterized coarse spaces for difficult sparse systems, limiting the efficiency gains obtained from learning. NeuraLSP addresses this gap by combining neural learning with an explicit low-rank spectral objective. Instead of learning prolongation solely through graph aggregation, NeuraLSP learns a low-rank representation of the left singular subspace associated with near-nullspace components of the system matrix $\mathbf{A}$. 

\textbf{SVD-Inspired Coarsening in Scientific Computing Applications.} \quad 
SVD-based near-nullspace compression has also appeared in domain-specific multigrid applications. For example, Whyte et al. \cite{Whyte2025} use SVD truncation of smoothed test vectors to construct prolongation and restriction matrices for multigrid preconditioning of Wilson fermions in lattice QCD. Related multigrid ideas have also been developed for structural mechanics and elasticity \cite{Vanek1999,Griebel2003}, incompressible flow and Navier--Stokes systems \cite{Knoll2000,Pernice2001,Elman2003}, and electromagnetics and Maxwell equations \cite{Reitzinger2002,Bochev2003,Hiptmair2007}. These applications further motivate the need for scalable, spectrally informed coarse-space construction.

\section{Mathematical Preliminaries}

\textbf{PDE Discretization and Linear Systems.} \quad 
Let $\Omega \subset \mathbb{R}^d$ be a bounded domain with boundary $\partial \Omega$. We consider a boundary-value problem in the form of 
\begin{equation}
    \label{generalized_pde}
    \left\{
    \begin{aligned}
        \mathcal{D}[u(\mathbf{x})] &= f(\mathbf{x}), 
        \quad && \mathbf{x}\in\Omega,\\
        u(\mathbf{x}) &=0, 
        \quad && \mathbf{x}\in\partial\Omega,
    \end{aligned}
    \right.
\end{equation}
where $\mathcal{D}$ is a differential operator and $f$ is a prescribed source term. In many scientific-computing applications, analytical solutions are unavailable, and the PDE must be solved numerically. After discretization, the problem is typically reduced to a sparse linear system
\begin{equation}
    \label{discretized_pde}
    \mathbf{A}\mathbf{u} = \mathbf{f},
\end{equation}
where $\mathbf{A}\in\mathbb{R}^{n\times n}$ is the discretized operator and 
$\mathbf{u},\mathbf{f}\in\mathbb{R}^n$ are the discrete solution and forcing vectors. The system matrix $\mathbf{A}$ depends on the underlying discretization scheme. Finite-difference methods are commonly used on structured grids \cite{Thomas1995}, finite-volume methods are often used for conservation-law formulations \cite{Leveque1992}, and finite-element methods are widely used for complex geometries and irregular meshes \cite{BrennerScott2008}.

In this work, we focus on the efficient solution of linear systems of the form Eq.~\eqref{discretized_pde}, where $\mathbf{A}$ is large and sparse. Such systems are often solved iteratively, and their convergence can depend strongly on the conditioning and spectral structure of $\mathbf{A}$. Preconditioning is therefore essential for reducing the number of iterations and improving overall solve time.

\textbf{Algebraic Multigrid and Smoothed Aggregation.} \quad 
Algebraic multigrid (AMG) methods construct a multilevel hierarchy directly from the matrix $\mathbf{A}$, without requiring explicit geometric information about the underlying grid or mesh. The matrix can be interpreted as a graph whose nodes correspond to unknowns and whose edges correspond to nonzero matrix entries. In classical AMG, this graph is used to identify strongly connected variables and to construct coarse levels that represent slowly varying error components.

The centerpiece in AMG is the prolongation operator $\mathbf{P}$, which maps coarse-level corrections to the fine level. Given a restriction operator $\mathbf{R}$, the Galerkin coarse operator is defined as: $\mathbf{A}_c = \mathbf{R}\mathbf{A}\mathbf{P}$. For symmetric positive definite problems, a common choice is $\mathbf{R}=\mathbf{P}^{\top}$. In a two-grid method, relaxation first reduces high-frequency error components, while coarse-grid correction targets algebraically smooth error components that relaxation alone removes slowly. Thus, the effectiveness of AMG depends on whether the range of $\mathbf{P}$ accurately captures these smooth error modes.

Smoothed aggregation AMG (SA-AMG) constructs $\mathbf{P}$ by first grouping fine-grid variables into aggregates and then building a tentative prolongation operator $\mathbf{T}$ from candidate near-nullspace vectors \cite{Vanek1996}. These candidate vectors encode error components that should be preserved on the coarse level. The tentative operator is then smoothed, commonly using one step of weighted Jacobi: $\mathbf{P}=\left(\mathbf{I} - \omega\mathbf{D}^{-1}\mathbf{A}\right)\mathbf{T}$, where $\mathbf{D}=\text{diag}(\mathbf{A})$ and $\omega$ is a relaxation parameter. This smoothing step reduces the energy of the tentative basis functions and allows the influence of each coarse variable to spread beyond its aggregate. The resulting prolongation operator is better aligned with the algebraically smooth components of the error.

\textbf{Near-Nullspace Vectors.} \quad
The near-nullspace of $\mathbf{A}$ consists of vectors that are difficult for relaxation to eliminate. For symmetric positive definite systems, these vectors can be characterized as low-energy modes, i.e., vectors $\mathbf{s}$ for which the Rayleigh quotient, $\left(\mathbf{s}^{\top}\mathbf{A}\mathbf{s}\right)/
    \left(\mathbf{s}^{\top}\mathbf{s}\right)$, is small. In multigrid terminology, such vectors are often called algebraically smooth error modes. Following ideas from bootstrap and adaptive AMG \cite{Brandt11}, near-nullspace samples can be generated by applying relaxation to random initial vectors. Let 
$\mathbf{S}^{(0)}\in\mathbb{R}^{n\times K}$ be a collection of random test vectors. Applying $s_1$ steps of weighted Jacobi gives $\mathbf{S}^{(s_1)}=\left(\mathbf{I}-\omega \mathbf{D}^{-1}\mathbf{A}\right)^{s_1}\mathbf{S}^{(0)}$.
Relaxation damps high-energy components of the random vectors, so the columns of 
$\mathbf{S}^{(s_1)}$ become enriched in smooth-error components. We denote the resulting smoothed-vector matrix by $\mathbf{S} \equiv \mathbf{S}^{(s_1)}$. The goal of the prolongation operator $\mathbf{P}$ is then to represent the dominant subspace of $\mathbf{S}$, so that coarse-grid correction can efficiently reduce the corresponding smooth errors.

\textbf{SVD-Centric Construction of Coarse Bases.} \quad
The smoothed-vector matrix $\mathbf{S}$ contains information about the near-nullspace components that should be represented on the coarse level. A natural way to extract this information is through the singular value decomposition $\mathbf{S} = \mathbf{U}\mathbf{\Sigma}\mathbf{V}^{\top}$. If we seek a rank-$k$ representation of $\mathbf{S}$, we can formulate the approximation problem: $\min_{\mathbf{P},\mathbf{W}}
    \left\|
        \mathbf{S} - \mathbf{P}\mathbf{W}
    \right\|_2,
    \quad
    \text{s.t.}
    \quad
    \mathbf{P}\in\mathbb{R}^{n\times k},\;
    \mathbf{P}^{\top}\mathbf{P}=\mathbf{I}_k,\;
    \mathbf{W}\in\mathbb{R}^{k\times K}$.
For fixed $\mathbf{P}$, the optimal coefficient matrix is 
$\mathbf{W}=\mathbf{P}^{\top}\mathbf{S}$. Therefore, the problem is equivalent to finding a $k$-dimensional subspace that best approximates the columns of $\mathbf{S}$. By the Eckart--Young--Mirsky theorem \cite{Eckart1936}, the optimal rank-$k$ approximation is $\mathbf{S}_k
    =
    \mathbf{U}_k
    \mathbf{\Sigma}_k
    \mathbf{V}_k^{\top}$,
where $\mathbf{U}_k$ contains the leading $k$ left singular vectors of $\mathbf{S}$. Thus, one optimal choice is
    $\mathbf{P} = \mathbf{U}_k,
    \quad
    \mathbf{W} = \mathbf{\Sigma}_k\mathbf{V}_k^{\top}$.
This factorization is not unique: for any orthogonal matrix 
$\mathbf{Q}\in\mathbb{R}^{k\times k}$, the pair $\mathbf{P} = \mathbf{U}_k\mathbf{Q},
    \quad
    \mathbf{W} = \mathbf{Q}^{\top}\mathbf{\Sigma}_k\mathbf{V}_k^{\top}$
achieves the same approximation. Therefore, only the subspace $\text{span}(\mathbf{P})$ is uniquely determined.

This SVD perspective connects near-nullspace recovery with low-rank subspace approximation. \cite{Chow06} uses this idea locally by computing SVDs of smoothed-error samples restricted to aggregates, which keeps the decomposition cost manageable. However, direct SVD computations can be expensive, and global singular vectors may be dense. This motivates our approach: instead of computing a per-instance SVD of the smoothed-vector matrix $\mathbf{S}$, we learn a map that produces an orthonormal low-rank coarse basis while preserving the key subspace-recovery property at global optima.

\section{Novel Approach}
\label{sec:method}

Consider a discretized PDE (e.g., Eq.~\eqref{discretized_pde}), we now focus on symmetric positive definite systems, for which Galerkin coarse-grid correction and preconditioned conjugate gradient methods are naturally applicable. NeuraLSP constructs a rank-controlled neural multigrid preconditioner from smoothed near-nullspace samples. The method consists of three steps: first, we generate a matrix of smoothed test vectors $\mathbf{S}$; second, a neural model predicts a low-rank basis for the dominant left singular subspace of $\mathbf{S}$; and third, this basis is used to define a Galerkin coarse correction inside a two-level preconditioner.

Unlike classical AMG, where the coarse-space dimension is determined by aggregation and strength-of-connection heuristics, NeuraLSP directly controls the coarse rank through a user-specified parameter $k$. This allows the method to retain the spectral information most relevant to smooth-error correction while avoiding unnecessarily large coarse representations.


\textbf{Novel Nested Left-Singular-Subspace (NLSS) Loss.} \quad 
Let $\mathbf{S}\in\mathbb{R}^{n\times K}$ denote a matrix of smoothed test vectors generated from random initial vectors by applying relaxation to the system matrix $\mathbf{A}$. The columns of $\mathbf{S}$ are enriched in algebraically smooth error components, and therefore provide samples of the near-nullspace modes that the coarse space should represent.

A natural objective for learning a $k$-dimensional subspace is to maximize the projected energy
\begin{equation}
    \label{eq:subspace_energy}
    J(\mathbf{P})
    =
    \left\|
        \mathbf{P}\mathbf{P}^{\top}\mathbf{S}
    \right\|_F^2
    =
    \operatorname{Tr}
    \left(
        \mathbf{P}^{\top}
        \mathbf{S}\mathbf{S}^{\top}
        \mathbf{P}
    \right),
    \qquad
    \mathbf{P}\in\operatorname{St}(n,k),
\end{equation}
where $\operatorname{St}(n,k)$ denotes the Stiefel manifold of $n\times k$ matrices with orthonormal columns. Maximizing Eq.~\eqref{eq:subspace_energy} recovers the leading $k$-dimensional left singular subspace of $\mathbf{S}$. However, this objective is invariant under transformations of the form $\mathbf{P}\mapsto \mathbf{P}\mathbf{Q}$ for any orthogonal $\mathbf{Q}\in\mathbb{R}^{k\times k}$. This invariance is appropriate when only the full rank-$k$ subspace is needed, but it does not impose any ordering on the learned basis. Consequently, the first $\ell<k$ columns of $\mathbf{P}$ need not span the best rank-$\ell$ subspace.

To obtain a rank-adaptive basis whose prefixes are spectrally meaningful, we introduce the \textbf{Nested Left-Singular-Subspace} loss, abbreviated as \textbf{NLSS}. Given a neural prediction $\widehat{\mathbf{P}}_\theta\in\mathbb{R}^{n\times k}$, we first compute an orthonormal basis, $\mathbf{P}_\theta= \operatorname{qf}\left(\widehat{\mathbf{P}}_\theta\right)$,
where $\operatorname{qf}(\cdot)$ denotes the thin-$Q$ factor from a QR decomposition. Let $\mathbf{P}_{\theta,1:\ell}$ denote the first $\ell$ columns of $\mathbf{P}_\theta$. We define
\begin{equation}
    \label{eq:nlss_loss}
    \mathcal{L}^{(k)}_{\mathrm{NLSS}}
    \left(
        \mathbf{S},
        \mathbf{P}_\theta
    \right)
    =
    \frac{1}{k}
    \sum_{\ell=1}^{k}
    \left(
        1
        -
        \frac{
            \left\|
                \mathbf{P}_{\theta,1:\ell}^{\top}
                \mathbf{S}
            \right\|_F^2
        }{
            \left\|
                \mathbf{S}
            \right\|_F^2
        }
    \right).
\end{equation}
Equivalently, if $\mathbf{P}_\theta=[\mathbf{p}_1,\ldots,\mathbf{p}_k]$, then minimizing Eq.~\eqref{eq:nlss_loss} is equivalent to maximizing:
$
    \label{eq:weighted_kyfan}
    \sum_{j=1}^{k}
    (k-j+1)
    \mathbf{p}_j^{\top}
    \mathbf{S}\mathbf{S}^{\top}
    \mathbf{p}_j$.
Thus, earlier columns receive a larger weight, encouraging the learned basis to recover the dominant singular directions in order. This nested structure allows the rank of the preconditioner to be reduced by truncating the learned basis without retraining the model. Each $\mathbf{S}_j$ is generated from a training instance. Importantly, this objective does not require ground-truth singular vectors. The loss is self-supervised: it uses only the smoothed near-nullspace samples.

The following theorem and corollary (proved in the appendices \ref{sec:theorem_proof} and \ref{sec:proof_corollary} respectively) show that NLSS recovers the dominant left singular subspace at global optima and that the nested weighting recovers an ordered basis when the leading singular values are simple. These results justify using the first $k$ learned columns as a rank-controlled coarse representation.

\begin{theorem}[Global minimizers of the NLSS objective]
\label{thm:nlss_subspace}
Let $\mathbf{S}\in\mathbb{R}^{n\times K}$ have singular value decomposition $\mathbf{S}=\mathbf{U}\mathbf{\Sigma}\mathbf{V}^{\top}$,
with singular values $\sigma_1\geq \sigma_2\geq \cdots$. Let
$\mathbf{P}\in\operatorname{St}(n,k)$.
Every global minimizer of
$\mathcal{L}^{(k)}_{\mathrm{NLSS}}(\mathbf{S},\mathbf{P})$
spans a dominant $k$-dimensional left singular subspace of $\mathbf{S}$. In particular, if
$\sigma_k>\sigma_{k+1}$, then every global minimizer satisfies $\operatorname{Range}(\mathbf{P})=\operatorname{Range}(\mathbf{U}_k)$,
where $\mathbf{U}_k$ contains the leading $k$ left singular vectors of $\mathbf{S}$.
\end{theorem}

\begin{corollary}[Ordered recovery under simple singular values]
\label{cor:nlss_ordered}
Under the assumptions of Theorem~\ref{thm:nlss_subspace}, suppose further that $\sigma_1>\sigma_2>\cdots>\sigma_k>\sigma_{k+1}$. Then any global minimizer of
$\mathcal{L}^{(k)}_{\mathrm{NLSS}}$
recovers the leading left singular vectors in order, up to column-wise signs: $\mathbf{P}=\mathbf{U}_k\mathbf{D}$,
where $\mathbf{D}$ is a diagonal matrix with entries in $\{\pm 1\}$.
\end{corollary}


\textbf{Neural Parameterization and Preconditioner Application.} \quad
The neural model receives the smoothed test-vector matrix $\mathbf{S}$ and outputs an unconstrained matrix
$
    \widehat{\mathbf{P}}_\theta
    =
    f_\theta(\mathbf{S})
    \in\mathbb{R}^{n\times k}.
$
Because the theoretical results assume an orthonormal basis, we apply a thin QR decomposition:
$
    \widehat{\mathbf{P}}_\theta
    =
    \mathbf{Q}
    \mathbf{R},
    \ 
    \mathbf{P}_\theta
    =
    \mathbf{Q}.
$
This QR layer ensures that
$\mathbf{P}_\theta^{\top}\mathbf{P}_\theta=\mathbf{I}_k$.
In our implementation, $f_\theta$ is a four-layer MLP with hidden widths
$128$--$256$--$256$--$128$, LayerNorm, and GELU activations. The output is reshaped into an $n\times k$ matrix before orthonormalization.

\begin{figure}[h!]
\includegraphics[clip, width=0.95\textwidth]{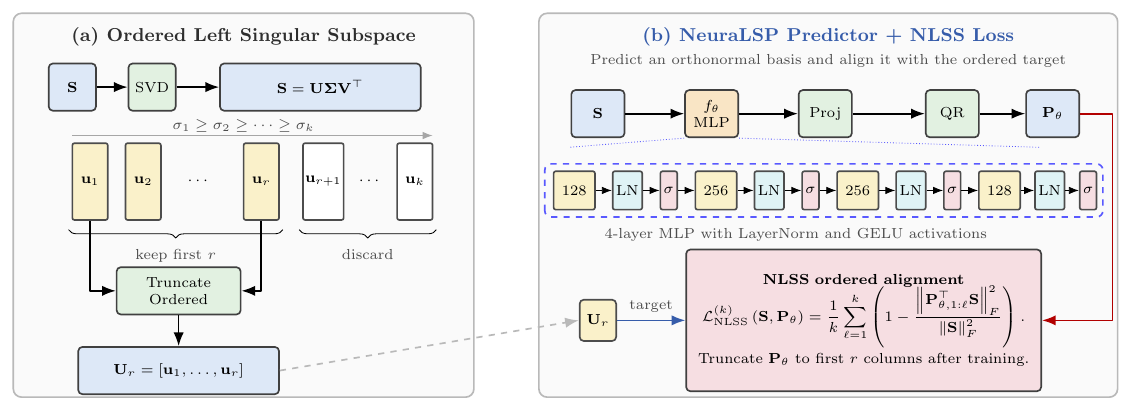}
\caption{Our NeuraLSP pipeline and architecture overview. The collection of vectors $\mathbf{S}$ is passed through a neural network consisting of a 4-layer MLP (widths 128-256-256-128) with LayerNorms and GELU activations in-between. Then, the output is projected onto a matrix, and QR-decomposed to ensure orthonormality. The result is the left singular subspace $\tilde{\mathbf{U}}$.} 
    \label{fig:pipeline_detailed}
\end{figure}

At inference time, NeuraLSP does not compute an SVD. Instead, it generates smoothed test vectors, applies the trained model, orthonormalizes the result, and forms the coarse operator. The setup phase is performed once for a given matrix $\mathbf{A}$, after which the resulting two-level operator is used as a preconditioner inside PCG. We illustrate how the NeuraLSP preconditioner works in Fig.~\ref{fig:rank_truncation} and how the size of the coarse space can be tuned. More practically, we provide the high-level NeuraLSP algorithm in Alg.~\ref{alg:neuralsp}. 

\begin{figure*}[h!]
    \centering
    \includegraphics[width=\textwidth]{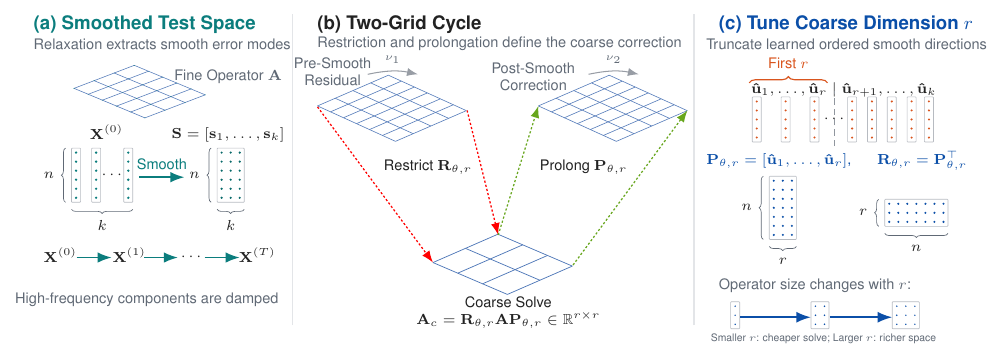}
    \caption{
    Overview of the NeuraLSP two-grid construction. Smoothed test vectors are used
    to learn ordered smooth directions. Truncating the first $r$ directions forms
    the tunable prolongation $\mathbf{P}_{\theta,r}$ and restriction
    $\mathbf{R}_{\theta,r}$, which define the coarse correction in the two-grid
    cycle.
    }
    \label{fig:rank_truncation}
\end{figure*}

\begin{algorithm}[h!]
  \caption{NeuraLSP setup and one preconditioner application}
  \label{alg:neuralsp}
  \begin{algorithmic}[1]
    \STATE {\bfseries Input:} SPD matrix $\mathbf{A}\in\mathbb{R}^{n\times n}$, trained model $f_\theta$, random test vectors $\mathbf{S}^{(0)}\in\mathbb{R}^{n\times K}$, rank $k$, smoothing parameter $\omega$, test-vector smoothing steps $s_1$, pre/post-smoothing steps $\nu_1,\nu_2$, coarse rank $r\leq k$.
    \STATE {\bfseries Setup:}
    \STATE $\mathbf{D}\leftarrow \operatorname{diag}(\mathbf{A})$
    \STATE $\mathbf{S}\leftarrow
    \left(
        \mathbf{I}
        -
        \omega
        \mathbf{D}^{-1}
        \mathbf{A}
    \right)^{s_1}
    \mathbf{S}^{(0)}$
    \STATE $\widehat{\mathbf{P}}_\theta\leftarrow f_\theta(\mathbf{S})$
    \STATE $\mathbf{P}_\theta\leftarrow \operatorname{qf}(\widehat{\mathbf{P}}_\theta)$
    \STATE $\mathbf{P}_{\theta,r} = \mathbf{P}_{\theta}[:,1:r]$
    \STATE $\mathbf{A}_c\leftarrow
    \mathbf{P}_{\theta,r}^{\top}
    \mathbf{A}
    \mathbf{P}_{\theta,r}$
    \STATE {\bfseries Apply preconditioner to residual $\mathbf{b}$:}
    \STATE $\mathbf{z}\leftarrow \mathbf{0}$
    \STATE Apply $\nu_1$ relaxation sweeps to approximately solve
    $\mathbf{A}\mathbf{z}=\mathbf{b}$
    \STATE $\boldsymbol{\rho}\leftarrow \mathbf{b}-\mathbf{A}\mathbf{z}$
    \STATE Solve
    $\mathbf{A}_c\mathbf{e}_c
    =
    \mathbf{P}_{\theta,r}^{\top}\boldsymbol{\rho}$
    \STATE $\mathbf{z}\leftarrow
    \mathbf{z}
    +
    \mathbf{P}_{\theta,r}\mathbf{e}_c$
    \STATE Apply $\nu_2$ post-smoothing relaxation sweeps to
    $\mathbf{A}\mathbf{z}=\mathbf{b}$, initialized at the current $\mathbf{z}$
    \STATE {\bfseries Return:} $\mathbf{z}\approx \mathbf{A}^{-1}\mathbf{b}$
  \end{algorithmic}
\end{algorithm}

When used with PCG, the two-level preconditioner should be implemented in an SPD-compatible form. In our experiments, we use the Galerkin restriction $\mathbf{R} =\mathbf{P}_\theta^{\top}$, which is the standard choice for SPD systems. For nonsymmetric systems, the same learned coarse-space construction can be combined with Krylov methods such as GMRES, but the convergence theory and symmetry requirements differ.

\section{Experiments}
\label{sec:experiments}

\subsection{Experimental Setup}

\subsubsection{Benchmarks}

We evaluate NLSS on five PDE-derived linear systems: diffusion, anisotropic
diffusion, screened Poisson, heat, and wave equations. Each problem is posed on
the bounded domain $\Omega=[0,1]^2$ and discretized using finite elements on
triangulated meshes. Unless otherwise stated, we use grids with $N=64$, giving
linear systems of dimension $n=(N+1)^2$, and collect $K=72$ smoothed test
vectors per instance. PDEs are further evaluated in the following two aspects: 

\textbf{Subspace-ordering Diagnostic.} \quad 
We first test whether NLSS learns an ordered left subspace. For this diagnostic,
we compare NLSS against the invariant subspace loss in Eq.~\eqref{eq:subspace_energy}.
Both models are trained to the maximum available rank using $K=32$ smoothed
test vectors on smaller systems with $N=9$. We then truncate the learned basis
to its first $r$ columns and measure the fraction of smoothed-vector energy
captured:
$
    \mathrm{Energy}(r)
    =
    \frac{\|\mathbf{Q}_{\theta,r}^\top\mathbf{S}\|_F^2}
    {\|\mathbf{S}\|_F^2},
$
where $\mathbf{Q}_{\theta,r}$ denotes the first $r$ learned basis vectors and
$\mathbf{S}\in\mathbb{R}^{n\times K}$ is the smoothed test-vector matrix.
Since the maximum rank is $K=32$, both losses can in principle recover the same
subspace at full rank. The goal of this diagnostic is to test whether NLSS
concentrates the most informative smooth directions in the earliest columns.

\textbf{Two-grid Solver Evaluation.} \quad
We next evaluate the learned subspaces inside a two-grid preconditioner for
PCG~\cite{Hestenes52}. Given an ordered learned basis, we form the prolongation
operator by truncating after the first $r$ directions,
$\mathbf{P}_r\in\mathbb{R}^{n\times r}$, and use the corresponding restriction
operator $\mathbf{R}_r=\mathbf{P}_r^T$. This makes the coarse dimension directly
tunable through $r$. For all methods, we use the same PCG tolerance
$\delta=10^{-6}$ and the same weighted Jacobi smoother with
$\nu_1=5$, $\nu_2=5$, and $\omega=0.66$.

\subsubsection{Baselines and Metrics}

\textbf{Baselines.} \quad
We compare NeuraLSP against both learned and classical baselines. Learned baselines
include a GNN trained with the loss of \cite{Luz20}, NeurKITT~\cite{Luo24}, Greenfeld et al.~\cite{Greenfeld19}, and an MLP trained on the classical subspace metric Eq.~\eqref{eq:subspace_energy}. The GNN is parameter-matched to our MLP, and is trained to the same maximum rank before truncation. Classical
baselines include SOR, incomplete Cholesky (ICC), SVD, RandomSVD, and smoothed
aggregation AMG using PyAMG~\cite{Bell2022}. Whenever applicable, all methods are evaluated at the same coarse dimension $n_c=r$, use the same two-grid solver, smoother, stopping tolerance, and use $K=72$ as the initial number of smoothed test vectors.

\textbf{Metrics.} \quad
We report solve time and total per-instance runtime in milliseconds. Solve time
measures the time spent in the PCG solve after the preconditioner has been
constructed. Total runtime includes all method-specific inference, setup,
collection of smoothed vectors when required, construction of the preconditioner,
and the complete PCG process until convergence. For total runtime, we report the
median. The interquartile range of runtimes [Q1, Q3] over the test set is reported in the appendix in Table \ref{tab:pde_runtime_quartiles}.
\subsection{Quantitative Results}
\label{sec:quant_results}

\textbf{Overview.} \quad 
We first report \textbf{NeuraLSP}'s efficiency gain on PDE solving using NLSS on multiple benchmarks as mentioned above. We then analyze the captured energy gap between NLSS and the subspace loss in Eq.~\eqref{eq:subspace_energy} as we reduce the rank $r$, also mentioned above.  For experimental completeness, we also report the scalability data up to $N=80$. We also conducted an ablation study on different coarse space sizes (i.e., different chosen ranks for our prolongation matrix). More detailed data are available in the Appendix. 

\textbf{PDE Solving Efficiency.} \quad 
Table~\ref{tab:pde_runtime} shows that we outspeed a wide variety of neural and classical baselines in terms of both solve time and total time across elliptic, parabolic, and hyperbolic PDEs. Notably, for the anisotropic equation, we can see that we improve the solve time by about 53\% and total end-to-end time by about 19\%. Similarly, for the wave equation, we improve the solve time by about 46\% and the total end-to-end time by about 15\%. Finally, for the heat equation, we see a solve time improvement of about 23\% and a total end-to-end improvement of about 17\%. Thus, we can experimentally see that amortizing the cost of SVD yields the benefits of a reduced coarse space for solving different families of PDEs, from the spectral properties of the smoothed vectors, without the computational overhead of SVD. Additionally, we compare iteration counts of NLSS vs. SVD in the appendix. 

\begin{table}[h!]
  \centering
  \caption{Runtime comparison (ms) of PDE solvers. The best performing method is \textbf{bolded}, and the second-best is \underline{underlined}. The improvement row reports the percent reduction of \textbf{NeuraLSP (Ours)} relative to the second-best method in each column.}
  \label{tab:pde_runtime}
  \resizebox{\textwidth}{!}{
  \begin{tabular}{lrrrrrrrrrr}
    \toprule
    & \multicolumn{2}{c}{Diffusion} & \multicolumn{2}{c}{Anisotropic} & \multicolumn{2}{c}{Screened Poisson} & \multicolumn{2}{c}{Heat} & \multicolumn{2}{c}{Wave} \\
    \cmidrule(lr){2-3} \cmidrule(lr){4-5} \cmidrule(lr){6-7} \cmidrule(lr){8-9} \cmidrule(lr){10-11}
    Method & Solve & Total & Solve & Total & Solve & Total & Solve & Total & Solve & Total \\
    \midrule
    Subspace & 73.24 & 139.03 & \underline{52.52} & 115.80 & 49.98 & 108.24 & 45.54 & 90.36 & \underline{61.36} & 111.11 \\
    GNN & 94.02 & 835.46 & 70.15 & 818.54 & 65.44 & 795.79 & 73.78 & 797.84 & 91.89 & 826.15 \\
    SA-AMG & 100.94 & 157.90 & 90.26 & 151.09 & 71.63 & 127.67 & 80.22 & 135.98 & 106.78 & 168.56 \\
    NeurKITT & \underline{48.98} & 139.44 & 54.29 & 145.54 & \underline{30.77} & 125.58 & \underline{33.92} & 129.45 & 116.32 & 204.64 \\
    Greenfeld & 277.85 & 808.73 & 209.75 & 749.73 & 160.51 & 717.75 & 147.87 & 676.70 & 254.54 & 781.26 \\
    SOR & 4054.68 & 4056.27 & 3122.29 & 3123.78 & 2211.74 & 2213.22 & 2512.36 & 2513.81 & 4134.43 & 4135.87 \\
    ICC & 2381.40 & 2507.00 & 2479.02 & 2592.75 & 1606.99 & 1725.16 & 1840.45 & 1937.05 & 2356.95 & 2479.86 \\
    SVD & 101.12 & 131.20 & 72.27 & \underline{106.23} & 68.57 & 100.66 & 60.37 & \underline{89.06} & 78.25 & \underline{110.67} \\
    RandomSVD & 79.92 & \underline{118.22} & 80.98 & 121.50 & 58.71 & \underline{97.91} & 54.43 & 91.90 & 82.74 & 116.70 \\
    \textbf{NeuraLSP (Ours)} & \textbf{35.78} & \textbf{99.35} & \textbf{24.41} & \textbf{85.87} & \textbf{28.68} & \textbf{83.89} & \textbf{26.19} & \textbf{74.03} & \textbf{33.18} & \textbf{93.87} \\
    \midrule
    \textit{Improvement} & \textbf{26.9\%} & \textbf{16.0\%} & \textbf{53.5\%} & \textbf{19.2\%} & \textbf{6.8\%} & \textbf{14.3\%} & \textbf{22.8\%} & \textbf{16.9\%} & \textbf{45.9\%} & \textbf{15.2\%} \\
    \bottomrule
  \end{tabular}
  }
\end{table}

Fig \ref{fig:cap_en} reports the energy gap $E_{\text{SVD}}(r) - E_{\text{model}}(r)$, where smaller values indicate closer agreement with the rank-$r$ SVD subspace. NLSS maintains a substantially smaller gap than the invariant subspace-loss baseline under prefix truncation, showing that the nested objective orders the learned basis more effectively. 

\begin{figure}

\includegraphics[width = 0.3\textwidth]{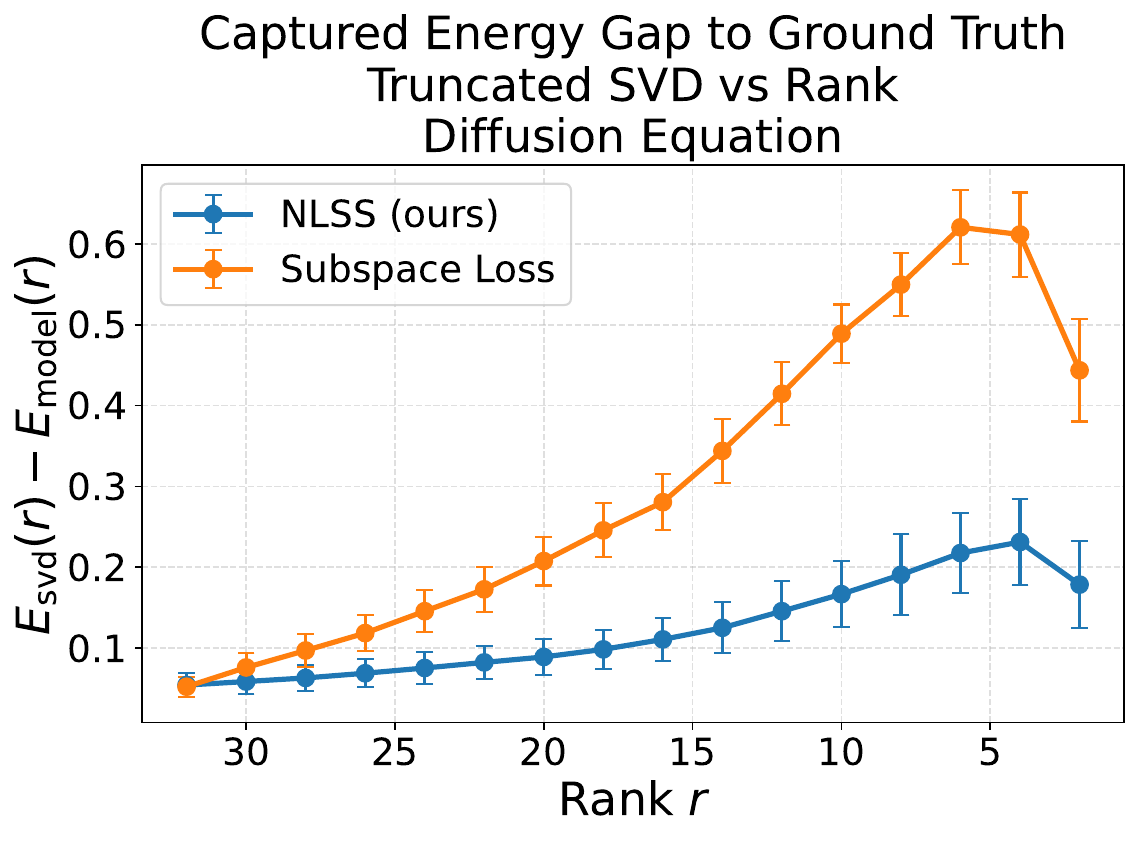}
\includegraphics[width=0.3\textwidth]{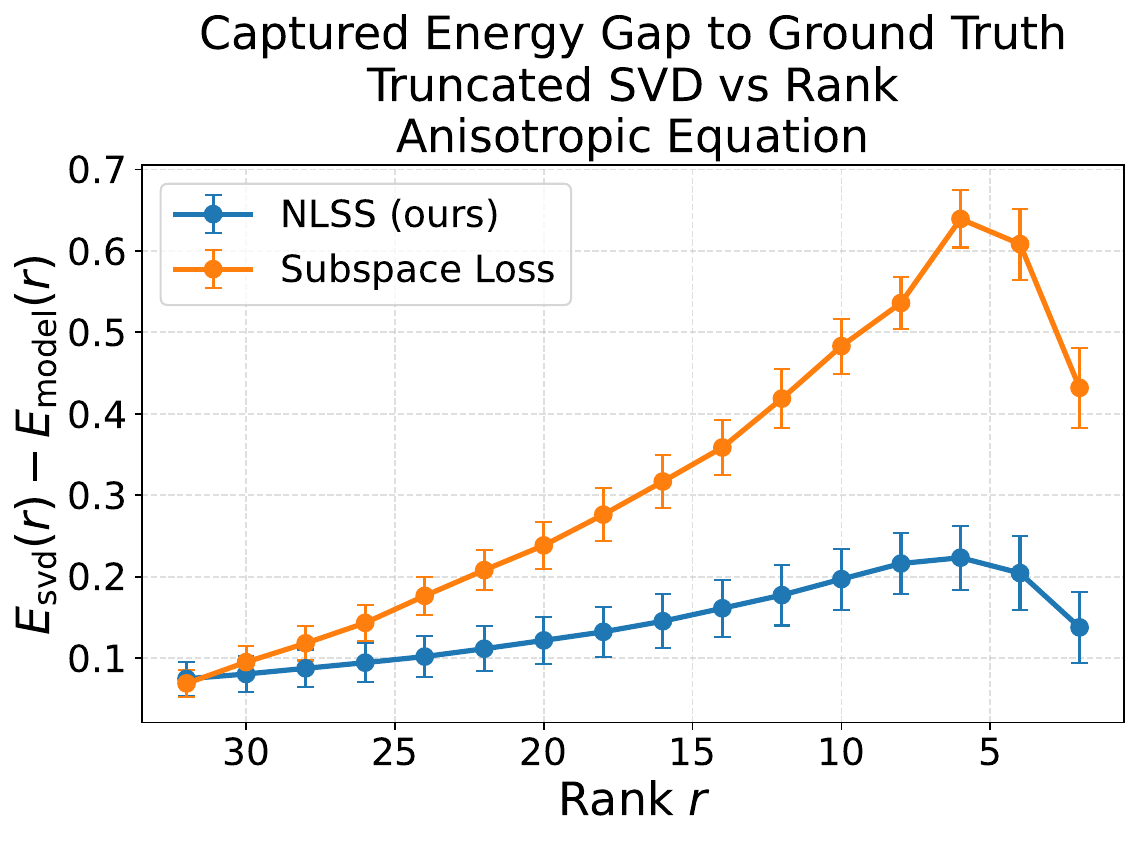}
\includegraphics[width=0.3\textwidth]{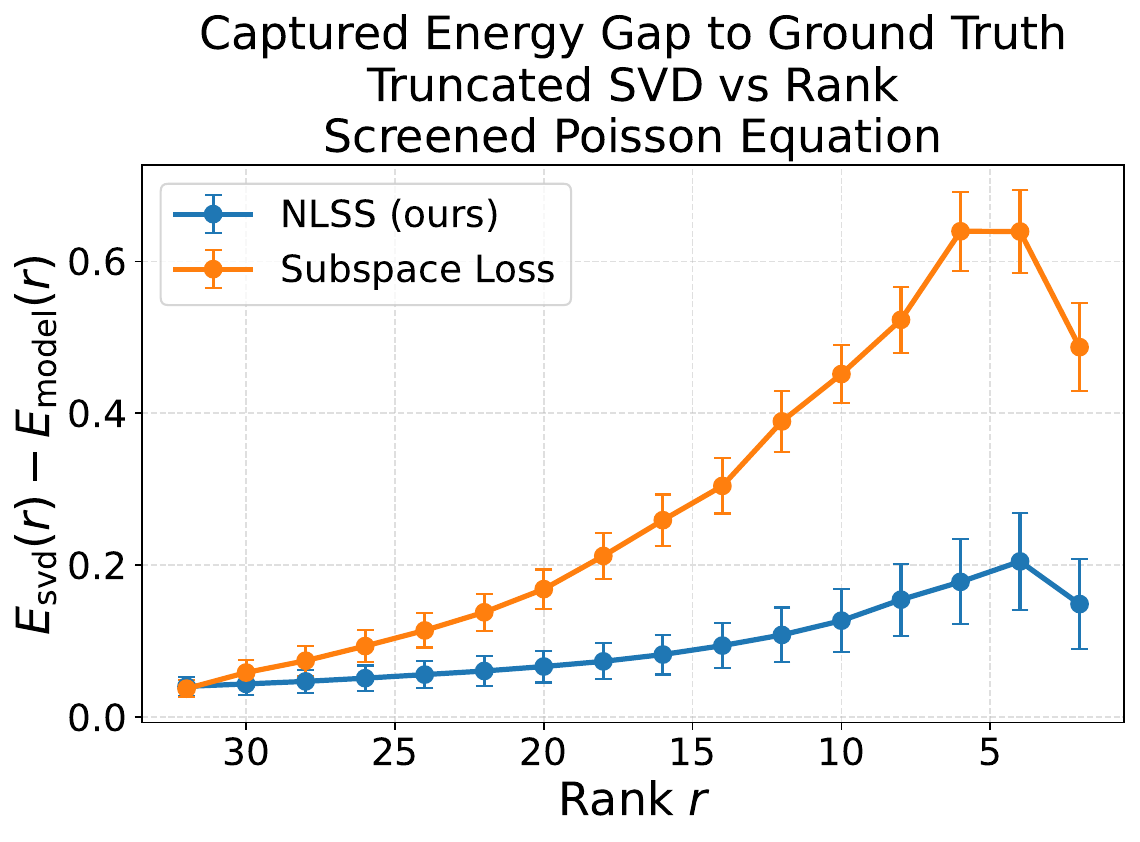}
\caption{$E_{\text{SVD}}(r) - E_{\text{model}}(r)$ for NLSS (blue) vs. subspace loss (orange) in comparison to SVD for diffusion (left), anisotropic (center), and screened Poisson (right) across 100 different samples. Error bars indicate one standard deviation over the test instances.}
\label{fig:cap_en}
\end{figure}

\textbf{Scalability.} \quad 
To see how well NLSS scales in runtime as we increase $N$, we compare NeuraLSP with SA-AMG and classic SVD. For each grid size $N$, we set $K=N$ and $r=K/2$ for both the NLSS and SVD, while SA-AMG automatically computes its own coarse space. (Note: ``End-to-End" time includes the generation of $\mathbf{A}$ given that we are working with different problem sizes in each experiment). 


\begin{table*}[h!]
\caption{
Ablation study comparing NLSS with SVD and SA-AMG for solving the diffusion equation on a triangulated mesh.
All times are in milliseconds. Lower is better for every metric. Bold denotes the best result, and underline denotes the second-best result within each $N$.
}
\label{tab:ablation}
\centering
\footnotesize
\setlength{\tabcolsep}{4pt}
\renewcommand{\arraystretch}{0.92}
\begin{tabular*}{\linewidth}{@{\extracolsep{\fill}}llrrrrr@{}}
\toprule
$N$ & Method & Solve & Setup & Inf. & E2E Avg. & E2E Std. \\
\midrule

16
& SVD
& 2.77
& \sbest{0.42}
& \sbest{0.96}
& \sbest{21.92}
& \best{12.12} \\
& SA-AMG ($n_c=31$)
& \best{1.95}
& 2.73
& 3.65
& 23.30
& 16.36 \\
& \textbf{NeuraLSP (Ours)}
& \sbest{2.75}
& \best{0.40}
& \best{0.88}
& \best{21.81}
& \sbest{12.34} \\

\addlinespace[2pt]
32
& SVD
& 10.89
& \best{0.64}
& \best{2.03}
& \best{78.06}
& 19.54 \\
& SA-AMG ($n_c=124$)
& \best{7.00}
& 2.77
& 4.19
& 78.99
& \best{16.11} \\
& \textbf{NeuraLSP (Ours)}
& \sbest{8.58}
& \sbest{2.47}
& \sbest{3.19}
& \sbest{78.73}
& \sbest{19.34} \\

\addlinespace[2pt]
48
& SVD
& \sbest{30.45}
& \best{1.09}
& \best{5.84}
& \sbest{195.79}
& \best{33.56} \\
& SA-AMG ($n_c=278$)
& 34.94
& 14.29
& \sbest{7.47}
& 206.35
& \sbest{39.98} \\
& \textbf{NeuraLSP (Ours)}
& \best{17.19}
& \sbest{3.50}
& 11.44
& \best{190.54}
& 49.31 \\

\addlinespace[2pt]
64
& SVD
& \sbest{88.82}
& \sbest{2.41}
& \best{13.58}
& \sbest{396.79}
& \sbest{57.09} \\
& SA-AMG ($n_c=491$)
& 113.79
& 46.80
& \sbest{13.85}
& 452.58
& \best{45.82} \\
& \textbf{NeuraLSP (Ours)}
& \best{39.17}
& \best{2.40}
& 15.58
& \best{349.14}
& 93.48 \\

\addlinespace[2pt]
80
& SVD
& \sbest{175.85}
& \best{5.38}
& 31.41
& \sbest{654.31}
& 101.64 \\
& SA-AMG ($n_c=770$)
& 294.80
& 143.65
& \best{23.38}
& 870.58
& \sbest{92.01} \\
& \textbf{NeuraLSP (Ours)}
& \best{73.45}
& \sbest{7.19}
& \sbest{26.16}
& \best{548.47}
& \best{64.77} \\

\bottomrule
\end{tabular*}
\end{table*}

From Table \ref{tab:ablation}, we can see that NLSS scales well in terms of end-to-end performance for solving PDEs as $N$ grows. For larger grids, NeuraLSP becomes increasingly competitive and eventually outperforms both SVD and SA-AMG in solve time and end-to-end runtime.

\textbf{Ablation Study.} \quad
In Table~\ref{tab:coarse_space_ablation}, we present additional results for the elliptic equations we presented, for methods that varied in coarse-space size. In particular, we present results for coarse-space sizes $n_c=36$ and $n_c=64$. We can see that we improve upon other tunable two-grid preconditioners by at least 46\% when we set $n_c=36$ and by about 33\% when $n_c=64$. 


\providecommand{\best}[1]{\textbf{#1}}
\providecommand{\sbest}[1]{\underline{#1}}

\begin{table*}[h!]
  \centering
  \caption{
    Coarse-space ablation on diffusion, anisotropic, and screened Poisson benchmarks.
    All times are reported in milliseconds. Lower is better. Bold denotes the best value
    within each equation and coarse-space size, and underline denotes the second-best value.
    The improvement row reports the percent reduction of \textbf{NeuraLSP (Ours)} relative to
    the second-best method in each column.
  }
  \label{tab:coarse_space_ablation}
  \resizebox{\textwidth}{!}{
  \begin{tabular}{lrrrrrrrrrrrr}
    \toprule
    & \multicolumn{4}{c}{Diffusion} & \multicolumn{4}{c}{Anisotropic} & \multicolumn{4}{c}{Screened Poisson} \\
    \cmidrule(lr){2-5} \cmidrule(lr){6-9} \cmidrule(lr){10-13}
    Method & Solve & Med. & $Q_1$ & $Q_3$ & Solve & Med. & $Q_1$ & $Q_3$ & Solve & Med. & $Q_1$ & $Q_3$ \\
    \midrule
    \multicolumn{13}{@{}l}{\textbf{Coarse-space size $n_c = 36$}} \\
    \addlinespace[2pt]
    Subspace  & \underline{72.11} & 137.55 & 119.77 & 156.53 & \underline{51.49} & 117.85 & 98.77 & 131.38 & \underline{48.87} & 107.49 & 94.07 & 122.44 \\
    GNN       & 89.14 & 826.26 & 806.61 & 858.97 & 59.32 & 802.03 & 774.50 & 824.66 & 58.08 & 783.10 & 767.03 & 818.18 \\
    SVD       & 91.07 & 122.39 & 109.11 & 137.13 & 61.67 & \underline{92.47} & \underline{81.10} & \underline{106.60} & 62.23 & 93.14 & 79.42 & 112.55 \\
    RandomSVD & 72.76 & \underline{104.43} & \underline{93.79} & \underline{123.74} & 74.62 & 110.13 & 94.58 & 131.32 & 53.78 & \underline{86.19} & \underline{75.93} & \underline{96.73} \\
    \textbf{NeuraLSP (Ours)} & \textbf{38.94} & \textbf{100.34} & \textbf{92.05} & \textbf{123.03} & \textbf{26.02} & \textbf{86.17} & \textbf{77.19} & \textbf{100.94} & \textbf{25.77} & \textbf{79.78} & \textbf{71.36} & \textbf{94.03} \\
    \addlinespace[2pt]
    \textit{Improvement} & \textbf{46.00\%} & \textbf{3.92\%} & \textbf{1.86\%} & \textbf{0.57\%} & \textbf{49.47\%} & \textbf{6.81\%} & \textbf{4.82\%} & \textbf{5.31\%} & \textbf{47.27\%} & \textbf{7.44\%} & \textbf{6.02\%} & \textbf{2.79\%} \\
    \midrule
    \multicolumn{13}{@{}l}{\textbf{Coarse-space size $n_c = 64$}} \\
    \addlinespace[2pt]
    Subspace  & \underline{63.23} & \underline{130.31} & \underline{117.10} & \underline{147.10} & \underline{48.75} & \underline{109.36} & \underline{96.87} & 129.73 & \underline{45.51} & \underline{101.96} & 94.76 & \underline{116.71} \\
    GNN       & 115.94 & 855.49 & 835.35 & 878.66 & 81.59 & 820.87 & 794.22 & 857.51 & 80.26 & 809.04 & 785.16 & 850.85 \\
    SVD       & 110.70 & 143.14 & 126.43 & 161.28 & 86.74 & 118.00 & 106.35 & \underline{128.73} & 77.39 & 109.44 & \underline{94.43} & 133.04 \\
    RandomSVD & 94.13 & 133.36 & 119.24 & 163.95 & 97.31 & 141.27 & 124.96 & 162.67 & 70.07 & 112.33 & 100.53 & 122.48 \\
    \textbf{NeuraLSP (Ours)} & \textbf{42.54} & \textbf{106.12} & \textbf{94.41} & \textbf{127.50} & \textbf{29.18} & \textbf{89.72} & \textbf{81.32} & \textbf{104.25} & \textbf{30.69} & \textbf{84.97} & \textbf{78.00} & \textbf{92.90} \\
    \addlinespace[2pt]
    \textit{Improvement} & \textbf{32.72\%} & \textbf{18.56\%} & \textbf{19.38\%} & \textbf{13.32\%} & \textbf{40.14\%} & \textbf{17.96\%} & \textbf{16.05\%} & \textbf{19.02\%} & \textbf{32.56\%} & \textbf{16.66\%} & \textbf{17.40\%} & \textbf{20.40\%} \\
    \bottomrule
  \end{tabular}
  }
\end{table*}

\section{Conclusions and Future Work}
\label{sec:conclusion}
This paper presents \textbf{NeuraLSP}, a novel neural spectral preconditioner framework for accelerating PDE solvers. NeuraLSP manifests a mathematically rigorous NLSS loss to identify the rank-deficient left singular subspace of near-nullspace vectors for a family of discretized PDEs. We have provided a theoretical guarantee of a global minimum for our novel NLSS loss. When combined with our NeuraLSP model, NLSS enables us to avoid partitioning the matrix into aggregates while capturing important global structures. We have shown that NeuraLSP outperforms both classical and neural methods for optimizing AMG in terms of runtime while maintaining comparable convergence behavior across several challenging PDE problems. 

\textbf{Limitations and Future Work.} \quad 
Although our approach is architecturally agnostic, there is no guarantee that other architectures will benefit to the same extent as the one we presented. GNNs, for example, may generalize to different discretization sizes, but lack the computational efficiency of a simple MLP. In future work,  we aim to assess whether this framework can be used as a module within a larger-scale architecture to better precondition iterative schemes and, in turn, reduce the number of iterations.

\nocite{*} 

\bibliographystyle{plain} 
\bibliography{mybibliography}

\medskip


\appendix
\newpage
\section{Appendix}

For the theoretical and experimental completeness of this paper, in this Appendix, we prove the Theorems in the main content, and also list our implementation details.

\subsection{von Neumann Trace Inequality}
For any $\mathbf{A},\mathbf{B}\in\mathbb{C}^{n\times n}$ with singular values $\alpha_i$ for $\mathbf{A}$ and $\beta_i$ for $\mathbf{B}$ sorted decreasingly, we have:
\[
|\text{Tr}(\mathbf{AB})|\leq \sum_{i=1}^n\alpha_i\beta_i.
\]

\subsection{Proof of Theorem 1}
\label{sec:theorem_proof}

\begin{proof}

First, we define $\mathcal{J}^{(k)}(\mathbf{S},\tilde{\mathbf{P}}) = \sum_{\ell=1}^k||\tilde{\mathbf{P}}_\ell^T\mathbf{S}||_F^2$. Then, minimizing $\mathcal{L}^{(k)}$ becomes equivalent to maximizing $\mathcal{J}^{(k)}$. Next, using the properties of the Frobenius norm, we have:
\[
||\tilde{\mathbf{P}}_\ell^T\mathbf{S}||_F^2 = \text{Tr}(\tilde{\mathbf{P}}_\ell^T\mathbf{SS}^T\tilde{\mathbf{P}}_\ell) = \sum_{j=1}^\ell \tilde{\mathbf{p}}_j^T\left(\mathbf{SS}^T\right)\mathbf{p}_j,
\]
\[
\Rightarrow \mathcal{J}^{(k)} = \sum_{\ell=1}^k\left(\sum_{j=1}^\ell \tilde{\mathbf{p}}_j^T(\mathbf{SS}^T)\tilde{\mathbf{p}}_j\right).
\]
Now, observe that $\tilde{\mathbf{p}}_1$ appears $k$ times, $\tilde{\mathbf{p}}_2$ appears $k-1$ times, etc. So, if we let $w_j \equiv k-j+1$, we can rewrite our loss as:
\[
\mathcal{J}^{(k)} = \sum_{j=1}^k w_j \tilde{\mathbf{p}}_j^T(\mathbf{SS}^T)\tilde{\mathbf{p}}_j.
\]
Now, let $\mathbf{A}\equiv \mathbf{SS}^T$, then, our objective becomes:
\[
\mathcal{J}^{(k)} = \sum_{j=1}^k w_j \tilde{\mathbf{p}}_j^T \mathbf{A}\tilde{\mathbf{p}}_j.
\]
Now, we let $\mathbf{W} := \text{diag}(w_1,\cdots, w_k)$. Then, we can again rewrite $\mathcal{J}^{(k)}$:
\[
\mathcal{J}^{(k)} = \text{Tr}(\mathbf{A}\tilde{\mathbf{P}}\mathbf{W}\tilde{\mathbf{P}}^T).
\]
We can now extend the matrix $\tilde{\mathbf{P}}$ as $\mathbf{Q} = [\tilde{\mathbf{P}}, \tilde{\mathbf{P}}_{\perp}]$ where $\mathbf{Q}$ is now an orthogonal matrix. Finally, we define one last element, $\mathbf{B}$:
\[
\mathbf{B}:= \mathbf{Q} \begin{bmatrix}
\mathbf{W} & 0\\
0 & 0
\end{bmatrix} \mathbf{Q}^T = \tilde{\mathbf{P}}\mathbf{W} \tilde{\mathbf{P}}^T,
\]
\[
\Rightarrow \mathcal{J}^{(k)} = \text{Tr}(\mathbf{AB}).
\]
Now, by the von Neumann trace inequality for positive semidefinite matrices \cite{von1937}, since the $w_i$'s are the eigenvalues of $\mathbf{B}$, we have:
\[
\mathcal{J}^{(k)} \leq \sum_{i=1}^k w_i \lambda_i(\mathbf{A}).
\]
Finally, if $\mathbf{B}$ shares eigenvectors with $\mathbf{A}$ (i.e., when $\mathbf{Q}=\mathbf{U}$, which is the eigenbasis of $\mathbf{A}$), then this bound becomes tight. In our case, the first $k$ columns of $\mathbf{Q}$ are exactly the top eigenvectors of $\mathbf{A}$, which correspond to the top $k$ left singular vectors of $\mathbf{S}$. Thus, if we choose $\tilde{\mathbf{P}} = \mathbf{U_k}$, we attain the upper bound and, in turn, maximize $\mathcal{J}^{(k)}$. Moreover, if $\mathbf{P}^*$ is any global maximizer of $\mathcal{J}^{(k)}$, then we still have $\mathcal{J}^{(k)} = \sum_{i=1}^k w_i\lambda_i(\mathbf{A})$, i.e., $\text{span}(\mathbf{P}^*)$ lies in the top-$k$ eigenspace of $\mathbf{SS}^T$.

\end{proof}

\subsection{Proof of Corollary 2}
\label{sec:proof_corollary}

\begin{proof}
Let $\mathbf{S} = \mathbf{U\Sigma V}^T$ and, like in the proof for Theorem~\ref{thm:nlss_subspace}, we can define $\mathbf{A}\equiv \mathbf{SS}^T =\mathbf{U}\mathbf{\Sigma}^2\mathbf{U}^T$. Let $\mathbf{P}\in \text{St}(n,k)$ and $\mathbf{W} = \text{diag}(w_1,\cdots,w_k)$ such that $w_j = k-j+1$. Following the steps of the proof for Theorem~\ref{thm:nlss_subspace}, we end up with:
\[
\text{Tr}(\mathbf{AB}) \leq \sum_{i=1}^k \lambda_i(\mathbf{A})w_i,
\]
where $\lambda_1(\mathbf{A})\geq \lambda_2(\mathbf{A})\geq \cdots$ are the eigenvalues of $\mathbf{A}$. This equality will hold if and only if $\mathbf{A}$ and $\mathbf{B}$ are simultaneously diagonalizable with aligned eigenvectors (up to rotations inside eigenspaces corresponding to repeated eigenvalues). Now, we assume that $\sigma_1>\sigma_2>\cdots>\sigma_k>\sigma_{k+1}$ where $k<m$. Then, since $\lambda_i(\mathbf{A}) = \sigma_i^2$, we also have $\lambda_1(\mathbf{A})> \lambda_k(\mathbf{A}) > \lambda_{k+1}(\mathbf{A})$. Thus, the eigenvectors $\mathbf{u_1},\cdots, \mathbf{u_k}$ of $\mathbf{A}$ are uniquely determined up to a sign. So, for the equality case to hold, $\mathbf{B}$ must share the eigenvectors with $\mathbf{A}$. That is, we must have the eigendecomposition:
\[
\mathbf{B} = \sum_{i=1}^k w_i\mathbf{u_i}\mathbf{u_i}^T.
\]
However, we also know that each column $\mathbf{p_j}$ of $\mathbf{P}$ is an eigenvector of $\mathbf{B}$ with eigenvalue $w_j$ and because $w_j$ is a simple eigenvalue of $\mathbf{B}$, its eigenvector is unique up to a sign. Therefore, $\mathbf{p_j}$ must equal $\pm \mathbf{u_j}$ for each $j=1,\cdots, k$, i.e.:
\[
\mathbf{P}^* = \mathbf{U_{k}D},
\]
for a diagonal sign matrix $\mathbf{D}$.

\end{proof}

\subsection{Proof of Optimality of Smooth Error Reduction}

\begin{theorem}
Let $\mathbf{S}\in\mathbb{R}^{n\times m}$ be the collection of smooth error vectors. Let $\mathcal{P}_k$ be the set of all orthogonal projection matrices of rank $k$. Then, the projector $\mathbf{PP}^T$ constructed from the global minimizer $\mathbf{P}^*$ of the NLSS loss captures the maximum possible Frobenius energy. 
\end{theorem}

\begin{proof}
Let $\mathbf{R}$ be the residual matrix after projecting the smooth vectors $\mathbf{S}$ onto the subspace spanned by $\mathbf{P}$:
\[
\mathbf{R} = \mathbf{S} - \mathbf{PP}^T\mathbf{S} = (\mathbf{I} - \mathbf{PP}^T)\mathbf{S}.
\]
We seek to minimize the total energy of the residual, defined by the squared Frobenius norm:
\[
\min_{\mathbf{P}\in\text{St}(n,k)}||(\mathbf{I}-\mathbf{PP}^T)\mathbf{S}||_F^2.
\]
Using the property that $||\mathbf{A}||_F^2 = \text{Tr}(\mathbf{A}^T\mathbf{A})$, we have:
\[
||\mathbf{R}||_F^2 = \text{Tr}(\mathbf{S}^T(\mathbf{I}-2\mathbf{PP}^T+\mathbf{PP}^T\mathbf{PP}^T)\mathbf{S}).
\]
Since $\mathbf{P}^T\mathbf{P} = \mathbf{I}$, we have:
\[
= \text{Tr}(\mathbf{S}^T\mathbf{S}) - \text{Tr}(\mathbf{S}^T\mathbf{PP}^T\mathbf{S}).
\]
Note that $\text{Tr}(\mathbf{S}^T\mathbf{S}) = ||\mathbf{S}||_F^2$ is constant, and by the cyclic property of trace, our problem becomes:
\[
\arg\min_\mathbf{P} ||\mathbf{R}||_F^2 \Leftrightarrow \arg\max_\mathbf{P} \text{Tr}(\mathbf{P}^T\mathbf{SS}^T\mathbf{P}).
\]
Now, we know that the NLSS loss provides us with $\mathbf{P}^*$ such that $\text{span}(\mathbf{P}^*) = \text{span}(\mathbf{U_k})$ where $\mathbf{U_k}$ are the top $k$ eigenvectors of $\mathbf{SS}^T$. According to Ky Fan's maximum principle, the function $f(\mathbf{P}) = \text{Tr}(\mathbf{P}^T\mathbf{AP})$ for a symmetric matrix $\mathbf{A}$, is maximized over the Stiefel manifold when the columns of $\mathbf{P}$ span the invariant subspace corresponding to the $k$ largest eigenvalues of $\mathbf{A}$. Therefore, $\mathbf{P}^*$ maximizes the captured energy $\text{Tr}(\mathbf{P}^T\mathbf{SS}^T\mathbf{P})$. Thus, $\mathbf{P}^*$ strictly minimizes the projection error $||\mathbf{S} - \mathbf{PP}^T\mathbf{S}||_F^2$. 

\end{proof}

\subsection{Galerkin Convergence}

\begin{definition}
The Galerkin projector is given as $$\mathbf{\Pi_A}(\mathbf{P}) = \mathbf{P}(\mathbf{P}^T\mathbf{A}\mathbf{P})^{-1}\mathbf{P}^T\mathbf{A}.$$
\end{definition}
\begin{definition}
The \emph{total post-correction smooth-energy error} is given as:
$$J(\mathbf{P}) := ||\mathbf{A}^{\frac{1}{2}}(\mathbf{I} - \mathbf{\Pi_A}(\mathbf{P}))\mathbf{S}||_F^2.$$
\end{definition}
\begin{theorem}
Let $\mathbf{A}$ be SPD and let $\mathbf{S} = [\mathbf{s}_1,\cdots,\mathbf{s}_m]\in\mathbb{R}^{n\times m}$ be sampled smooth error vectors. For full-rank $\mathbf{P}\in\mathbb{R}^{n\times k}$, among rank-$k$ coarse spaces, $J(\mathbf{P})$ is minimized when the range of $\mathbf{A}^{\frac{1}{2}}\mathbf{P}$ equals the top-$k$ left singular subspace of $\mathbf{A}^{\frac{1}{2}}\mathbf{S}$.
\end{theorem}
\begin{proof}
Let $\mathbf{Z}:=\mathbf{A}^{\frac{1}{2}}\mathbf{S}$ and $\mathbf{Q}:=\mathbf{A}^{\frac{1}{2}}\mathbf{P}$. Then, we have:
$$\mathbf{\Pi_A}(\mathbf{P}) = \mathbf{A}^{-\frac{1}{2}}\mathbf{Q}(\mathbf{Q}^T\mathbf{Q})^{-1}\mathbf{Q}^T\mathbf{A}^{\frac{1}{2}}$$ $$:= \mathbf{A}^{-\frac{1}{2}}\mathbf{\Pi}(\mathbf{Q})\mathbf{A}^{\frac{1}{2}}.$$ 
Then, we can rewrite the objective of $J(\mathbf{P})$ as: $$J(\mathbf{P}) = ||(\mathbf{I} - \mathbf{\Pi(\mathbf{Q}))\mathbf{Z}}||_F^2.$$
Now, we know that for orthogonal projector $\mathbf{\Pi(\mathbf{Q})}$, we have: 
$$||(\mathbf{I} - \mathbf{\Pi(\mathbf{Q}))\mathbf{Z}}||_F^2 = ||\mathbf{Z}||_F^2 - ||\mathbf{\Pi(Q)Z}||_F^2.$$
Thus, minimizing $\mathbf{J}$ is equivalent to maximizing $||\mathbf{\Pi(Q)Z}||_F^2$. Similar to the justification in A.2, we know that the maximizing $k$-dimensional subspace is the span of the top-$k$ left singular vectors of $\mathbf{Z} = \mathbf{A}^{\frac{1}{2}}\mathbf{S}$. 
\end{proof}
It is important to note that we do not sample the left singular vectors from $\mathbf{A}^{\frac{1}{2}}\mathbf{S}$ in our work. 

\subsection{Additional Experimental Details}
\label{app:experimental_details}

\paragraph{Mesh Generation.}
For each instance, we begin with an $(N+1)\times(N+1)$ grid on
$\Omega=[0,1]^2$. Boundary nodes are fixed, while interior nodes are randomly
jittered subject to rejecting degenerate meshes. We then apply Delaunay
triangulation and retain triangulated meshes for the finite-element
discretization. We train on 1000 samples and evaluate on 100 held-out samples.

\paragraph{NLSS Architecture and Optimization.}
For the PDE-solving experiments, NLSS uses a four-layer MLP with hidden widths
128, 256, 256, and 128, GELU activations~\cite{Hendrycks16}, and LayerNorm
between hidden layers. The output dimension is $nr$, corresponding to the
entries of $\mathbf{P}_r\in\mathbb{R}^{n\times r}$. We train for 1000 epochs
using Adam with learning rate $10^{-3}$ and PyTorch's default He initialization. 

\paragraph{Subspace-Ordering Diagnostic Architecture.}
For the smaller captured-energy diagnostic, we use systems with $N=9$ and
$K=32$ smoothed vectors. Both NLSS and the invariant subspace-loss baseline are
trained to full rank. We use a smaller MLP with one hidden layer of width 128 and
GELU activation.

\paragraph{GNN Baseline.}
The GNN baseline is trained using the loss of \cite{Luz20}. We choose its width
so that it has approximately the same number of trainable parameters as the NLSS MLP. The GNN uses learning rate $10^{-4}$, while the MLP uses learning rate
$10^{-3}$.

\paragraph{Hardware.}
All models are trained and evaluated on the same Apple M1 CPU environment. The peak resident set size observed during training was 1.94GB.

\paragraph{Implementation Details.}
NLSS is architecture-agnostic; in our experiments, we use a global MLP to reduce
inference overhead. The network outputs the entries of
$\mathbf{P}_r\in\mathbb{R}^{n\times r}$ and is trained with column-permutation
augmentation of the smoothed vectors so that the learned representation does not
depend on the arbitrary ordering of the test vectors. Training time did not exceed 20 minutes for any of the PDE families. For PyAMG, we use version 5.3 and are in agreement with the MIT License. 

\subsection{FEM Dataset Generation}
\label{app:fem_dataset_generation}

\paragraph{Common Mesh and Finite-Element Setup.}
All PDE benchmarks are posed on the unit square
\[
    \Omega = [0,1]^2
\]
with homogeneous Dirichlet boundary conditions. For a grid parameter $N$, we
begin with an $(N+1)\times (N+1)$ point cloud with spacing $h=1/N$, giving
$n=(N+1)^2$ vertices. The initial grid vertices are
\[
    \mathbf{p}_{ij}
    =
    \left(\frac{i}{N},\frac{j}{N}\right),
    \qquad i,j=0,\ldots,N.
\]
For the unstructured meshes, boundary vertices are fixed and each interior
vertex $\mathbf{p}_i$ is randomly perturbed:
\[
    \widetilde{\mathbf{p}}_i
    =
    \mathbf{p}_i + \boldsymbol{\xi}_i,
    \qquad
    \boldsymbol{\xi}_i
    \sim
    \mathrm{Unif}([-\epsilon_i,\epsilon_i]^2),
\]
where
\[
    \epsilon_i = \min\{0.35h,\;0.49d_i\},
\]
and $d_i$ is the distance from $\mathbf{p}_i$ to the boundary. We then apply
Delaunay triangulation to the jittered point cloud and remove triangles with
near-zero signed area.

On each triangle $T\in\mathcal{T}_h$, we use linear $P_1$ finite elements with
local basis functions $\{\phi_1,\phi_2,\phi_3\}$. Since the basis functions are
affine on each triangle, their gradients are constant on $T$. We denote these
gradient vectors by
\[
    \mathbf{g}_i^{(T)}
    =
    \nabla \phi_i|_T
    \in \mathbb{R}^2.
\]
For a scalar diffusion coefficient $\kappa_T$, the local stiffness matrix is
\[
    \left(\mathbf{K}^{(T)}\right)_{ij}
    =
    |T|\,\kappa_T
    \left(\mathbf{g}_i^{(T)}\right)^\top
    \mathbf{g}_j^{(T)}.
\]
For an anisotropic tensor coefficient
$\boldsymbol{\mathsf{K}}_T\in\mathbb{R}^{2\times 2}$, we instead use
\[
    \left(\mathbf{K}^{(T)}\right)_{ij}
    =
    |T|
    \left(\mathbf{g}_i^{(T)}\right)^\top
    \boldsymbol{\mathsf{K}}_T
    \mathbf{g}_j^{(T)}.
\]
The local mass matrix is
\[
    \mathbf{M}^{(T)}
    =
    \frac{|T|}{12}
    \begin{bmatrix}
    2 & 1 & 1 \\
    1 & 2 & 1 \\
    1 & 1 & 2
    \end{bmatrix}.
\]
The global stiffness matrix $\mathbf{K}$ and mass matrix $\mathbf{M}$ are
assembled by summing local contributions over all triangles. Homogeneous
Dirichlet boundary conditions are imposed by zeroing the rows and columns of the
global matrix corresponding to boundary vertices and setting the corresponding
diagonal entries to one.

\subsubsection{Diffusion Equation}

The diffusion benchmark is
\begin{equation}
    \label{eq:diffusion}
    \left\{
        \begin{aligned}
            -\nabla \cdot \left(\kappa(\mathbf{x})\nabla u(\mathbf{x})\right)
            &= f(\mathbf{x}),
            \qquad \mathbf{x}\in\Omega, \\
            u(\mathbf{x}) &= 0,
            \qquad \mathbf{x}\in\partial\Omega .
        \end{aligned}
    \right.
\end{equation}
The weak form is to find $u\in H_0^1(\Omega)$ such that, for all
$v\in H_0^1(\Omega)$,
\[
    a(u,v)
    =
    \int_\Omega
    \kappa \nabla u\cdot \nabla v\,d\mathbf{x}
    =
    \int_\Omega
    f v\,d\mathbf{x}.
\]
We take $\kappa$ to be piecewise constant on each triangle. For each triangle
$T$, we sample
\[
    \kappa_T = \exp(\eta_T),
    \qquad
    \eta_T \sim \mathcal{N}(0,1.5^2).
\]
The resulting matrix is the assembled stiffness matrix
\[
    \mathbf{A}_{\mathrm{diff}}
    =
    \mathbf{K}(\kappa).
\]

\subsubsection{Anisotropic Equation}

The anisotropic benchmark is
\begin{equation}
    \label{eq:anisotropic}
    \left\{
        \begin{aligned}
            -\nabla\cdot
            \left(
            \boldsymbol{\mathsf{K}}(\mathbf{x})
            \nabla u(\mathbf{x})
            \right)
            &= f(\mathbf{x}),
            \qquad \mathbf{x}\in\Omega, \\
            u(\mathbf{x}) &= 0,
            \qquad \mathbf{x}\in\partial\Omega .
        \end{aligned}
    \right.
\end{equation}
Here
$\boldsymbol{\mathsf{K}}(\mathbf{x})\in\mathbb{R}^{2\times 2}$ is a symmetric
positive definite tensor. On each triangle $T$, we sample an independent angle
\[
    \theta_T \sim \mathrm{Unif}(0,\pi),
\]
and define the rotation matrix
\[
    \mathbf{R}_T
    =
    \begin{bmatrix}
    \cos\theta_T & -\sin\theta_T \\
    \sin\theta_T & \cos\theta_T
    \end{bmatrix}.
\]
The local anisotropic tensor is then
\[
    \boldsymbol{\mathsf{K}}_T
    =
    \mathbf{R}_T
    \begin{bmatrix}
    10^3 & 0 \\
    0 & 1
    \end{bmatrix}
    \mathbf{R}_T^\top .
\]
Thus each local tensor has principal diffusivities differing by a factor of
$10^3$, with a randomly oriented principal direction. The resulting matrix is
\[
    \mathbf{A}_{\mathrm{aniso}}
    =
    \mathbf{K}(\boldsymbol{\mathsf{K}}).
\]

\subsubsection{Screened Poisson Equation}

The screened Poisson benchmark is
\begin{equation}
    \label{eq:screened_poisson}
    \left\{
        \begin{aligned}
            -\nabla \cdot
            \left(\kappa(\mathbf{x})\nabla u(\mathbf{x})\right)
            + \alpha u(\mathbf{x})
            &= f(\mathbf{x}),
            \qquad \mathbf{x}\in\Omega, \\
            u(\mathbf{x}) &= 0,
            \qquad \mathbf{x}\in\partial\Omega .
        \end{aligned}
    \right.
\end{equation}
The weak form is to find $u\in H_0^1(\Omega)$ such that, for all
$v\in H_0^1(\Omega)$,
\[
    a(u,v)
    =
    \int_\Omega
    \kappa \nabla u\cdot\nabla v\,d\mathbf{x}
    +
    \alpha
    \int_\Omega
    u v\,d\mathbf{x}.
\]
We sample a piecewise constant diffusion coefficient
\[
    \kappa_T = \exp(\eta_T),
    \qquad
    \eta_T\sim\mathcal{N}(0,1),
\]
and a scalar reaction coefficient
\[
    \alpha = 10^\zeta,
    \qquad
    \zeta\sim\mathrm{Unif}(0,2).
\]
Thus $\alpha\in[1,100]$. The resulting matrix is
\[
    \mathbf{A}_{\mathrm{sp}}
    =
    \mathbf{K}(\kappa) + \alpha \mathbf{M}.
\]
Since $\kappa_T>0$ and $\alpha>0$, the resulting operator is symmetric positive
definite after applying Dirichlet boundary conditions.

\subsubsection{Heat Equation}

The heat benchmark corresponds to one implicit time step for a heterogeneous
diffusion equation,
\[
    u_t - \nabla\cdot\left(\kappa(\mathbf{x})\nabla u(\mathbf{x})\right)
    =
    f(\mathbf{x}).
\]
Using backward Euler in time and linear finite elements in space gives
\[
    \left(\mathbf{M}+\Delta t\,\mathbf{K}(\kappa)\right)
    \mathbf{u}^{m+1}
    =
    \mathbf{M}\mathbf{u}^{m}
    +
    \Delta t\,\mathbf{f}^{m+1}.
\]
In our benchmark, we use only the left-hand-side matrix
\[
    \mathbf{A}_{\mathrm{heat}}
    =
    \mathbf{M}+\Delta t\,\mathbf{K}(\kappa).
\]
The diffusion coefficient is sampled independently per triangle as
\[
    \kappa_T = \exp(\eta_T),
    \qquad
    \eta_T\sim\mathcal{N}(0,1),
\]
and the timestep is sampled log-uniformly:
\[
    \Delta t = 10^\tau,
    \qquad
    \tau\sim\mathrm{Unif}[-2,0].
\]
Therefore $\Delta t\in[10^{-2},1]$. This produces a family of symmetric positive
definite matrices whose relative mass and stiffness contributions vary across
samples.

\subsubsection{Wave Equation}

The wave benchmark corresponds to an implicit step for a heterogeneous
wave-type equation,
\[
    u_{tt}
    -
    c^2
    \nabla\cdot
    \left(\kappa(\mathbf{x})\nabla u(\mathbf{x})\right)
    =
    f(\mathbf{x}).
\]
The resulting left-hand-side matrix has the form
\[
    \mathbf{A}_{\mathrm{wave}}
    =
    \mathbf{M}
    +
    (c\Delta t)^2 \mathbf{K}(\kappa).
\]
As in the heat benchmark, $\mathbf{M}$ is the finite-element mass matrix and
$\mathbf{K}(\kappa)$ is the heterogeneous stiffness matrix. For the wave
benchmark, the diffusion coefficient is sampled with larger contrast:
\[
    \kappa_T = \exp(\eta_T),
    \qquad
    \eta_T\sim\mathcal{N}(0,1.5^2).
\]
The wave speed and timestep are sampled as
\[
    c = 10^\gamma,
    \qquad
    \gamma\sim\mathrm{Unif}[2,3],
\]
and
\[
    \Delta t\sim\mathrm{Unif}[h,2h],
    \qquad
    h=\frac{1}{N}.
\]
Thus, the stiffness scaling factor is $(c\Delta t)^2$. Since
$c\in[10^2,10^3]$, the stiffness term often dominates the mass term, producing
matrices with strongly elliptic behavior while still arising from an implicit
wave discretization.

\subsection{Quartiles of Runtimes for PDE Solving Experiments}

In Table~\ref{tab:pde_runtime_quartiles}, we present the quartiles of the runtimes for the experiments shown in the PDE Solving Efficiency section of~\ref{sec:quant_results}. 

\begin{table}[h!]
  \centering
  \caption{Extended runtime comparison: First (Q1) and Third (Q3) Quartiles in milliseconds (ms). The best performing method is \textbf{bolded}, and the second-best is \underline{underlined}.}
  \label{tab:pde_runtime_quartiles}
  \resizebox{\textwidth}{!}{
  \begin{tabular}{lrrrrrrrrrr}
    \toprule
    & \multicolumn{2}{c}{Diffusion} & \multicolumn{2}{c}{Anisotropic} & \multicolumn{2}{c}{Screened Poisson} & \multicolumn{2}{c}{Heat} & \multicolumn{2}{c}{Wave} \\
    \cmidrule(lr){2-3} \cmidrule(lr){4-5} \cmidrule(lr){6-7} \cmidrule(lr){8-9} \cmidrule(lr){10-11}
    Method & Q1 & Q3 & Q1 & Q3 & Q1 & Q3 & Q1 & Q3 & Q1 & Q3 \\
    \midrule
    Subspace & 121.15 & 166.87 & 99.34 & 143.96 & 92.55 & 121.82 & 82.90 & \underline{93.34} & \underline{104.51} & \underline{122.73} \\
    GNN & 816.77 & 865.94 & 790.12 & 844.76 & 774.08 & 820.39 & 787.79 & 814.78 & 806.41 & 844.89 \\
    SA-AMG & 146.94 & 174.08 & 138.76 & 180.21 & 117.97 & 155.77 & 125.56 & 170.70 & 146.76 & 198.03 \\
    NeurKITT & 127.42 & 166.68 & 135.58 & 167.42 & 111.82 & 158.38 & 115.04 & 149.47 & 172.69 & 245.37 \\
    Greenfeld & 779.01 & 852.28 & 718.76 & 789.34 & 683.31 & 765.56 & 655.57 & 697.16 & 755.68 & 824.85 \\
    SOR & 4017.27 & 4120.11 & 3058.68 & 3177.11 & 2085.94 & 2729.32 & 2196.82 & 2719.94 & 4095.89 & 4189.04 \\
    ICC & 2399.68 & 2592.23 & 2537.39 & 2675.52 & 1655.26 & 2145.76 & 1766.50 & 2038.35 & 2431.88 & 2517.00 \\
    SVD & 120.93 & 149.92 & \underline{91.08} & \underline{123.64} & \underline{84.25} & 117.44 & \underline{76.70} & 109.88 & 108.75 & 133.53 \\
    RandomSVD & \underline{106.51} & \underline{143.24} & 104.68 & 141.07 & 88.95 & \underline{106.09} & 87.11 & 94.51 & 109.65 & 127.00 \\
    \textbf{NeuraLSP (Ours)} & \textbf{88.55} & \textbf{112.79} & \textbf{76.07} & \textbf{103.95} & \textbf{73.22} & \textbf{97.41} & \textbf{70.80} & \textbf{79.68} & \textbf{78.88} & \textbf{104.43} \\
    \midrule
    \textit{Improvement} & \textbf{16.9\%} & \textbf{21.3\%} & \textbf{16.5\%} & \textbf{15.9\%} & \textbf{13.1\%} & \textbf{8.2\%} & \textbf{7.7\%} & \textbf{14.6\%} & \textbf{24.5\%} & \textbf{14.9\%} \\
    \bottomrule
  \end{tabular}
  }
\end{table}

\subsection{Iteration Counts of NeuraLSP vs. SVD}

Although standard SVD is optimal for Euclidean low-rank approximation of the sampled smoothed-vector matrix S, this objective does not directly optimize the convergence of the full smoothed two-grid preconditioner used inside PCG. To examine this distinction, Table~\ref{tab:iter_counts} reports median PCG iteration counts for NLSS and the standard SVD coarse basis at coarse dimension $n_c=48$. Both methods use the same smoothed-vector generation procedure, rank, Galerkin restriction $R=P^\top$, weighted Jacobi smoother, and stopping tolerance. NLSS requires fewer PCG iterations across all PDE families, suggesting that the learned basis yields a more effective full two-grid preconditioner even when standard SVD has stronger sample-wise reconstruction properties. A theoretical characterization of this phenomenon is left for future work. Iteration counts are medians over 100 test instances; all runs reached relative residual $10^{-6}$ without hitting the maximum iteration cap.

\begin{table}[h!]
\centering
\caption{Median PCG iterations for NLSS and standard SVD at $n_c=48$. Lower is better. Both methods use the same two-grid preconditioner, smoother, rank, and PCG tolerance $\delta = 10^{-6}$.}
\begin{tabular}{lcc}
\toprule
\textbf{PDE} & \textbf{NLSS Iterations} & \textbf{SVD Iterations} \\
\midrule
Diffusion & 27 & 64 \\
Anisotropic  & 21 & 55 \\
Screened Poisson  & 19 & 46 \\
Heat  & 17.5 & 45 \\
Wave  & 49.5 & 64 \\
\bottomrule
\label{tab:iter_counts}
\end{tabular}

\label{tab:pde_iterations}
\end{table}

\subsection{Two-Level Subspace Correction}

\label{sec:method:twogrid}

Given a learned orthonormal basis $\mathbf{P}_\theta\in\mathbb{R}^{n\times k}$, NeuraLSP defines the Galerkin coarse operator
\[
    \label{eq:coarse_operator}
    \mathbf{A}_c
    =
    \mathbf{P}_\theta^{\top}
    \mathbf{A}
    \mathbf{P}_\theta.
\]
For SPD $\mathbf{A}$ and full-rank $\mathbf{P}_\theta$, the matrix $\mathbf{A}_c$ is also SPD. The corresponding coarse-grid correction operator is
\[
    \label{eq:coarse_correction}
    \mathbf{C}_\theta
    =
    \mathbf{P}_\theta
    \mathbf{A}_c^{-1}
    \mathbf{P}_\theta^{\top}.
\]
Ignoring smoothing for the moment, the error propagation matrix for the coarse correction is
\[
    \label{eq:error_propagation}
    \mathbf{E}_c
    =
    \mathbf{I}
    -
    \mathbf{C}_\theta
    \mathbf{A}.
\]
If the current error lies exactly in the learned coarse space, i.e.,
$\mathbf{e}=\mathbf{P}_\theta\mathbf{y}$ for some $\mathbf{y}$, then
\[
    \mathbf{E}_c\mathbf{e}
    =
    \mathbf{P}_\theta\mathbf{y}
    -
    \mathbf{P}_\theta
    \left(
        \mathbf{P}_\theta^{\top}
        \mathbf{A}
        \mathbf{P}_\theta
    \right)^{-1}
    \mathbf{P}_\theta^{\top}
    \mathbf{A}
    \mathbf{P}_\theta
    \mathbf{y}
    =
    \mathbf{0}.
\]
Thus, Galerkin correction exactly eliminates error components contained in
$\operatorname{Range}(\mathbf{P}_\theta)$. More generally, the coarse correction computes the best approximation to the error from
$\operatorname{Range}(\mathbf{P}_\theta)$ in the $\mathbf{A}$-norm:
\[
    \label{eq:a_norm_projection}
    \left\|
        \mathbf{E}_c\mathbf{e}
    \right\|_{\mathbf{A}}
    =
    \min_{\mathbf{y}\in\mathbb{R}^k}
    \left\|
        \mathbf{e}
        -
        \mathbf{P}_\theta\mathbf{y}
    \right\|_{\mathbf{A}}.
\]
Therefore, the quality of the preconditioner depends on how well the learned coarse space captures the smooth error components that remain after relaxation.

The role of NLSS is to construct such a space from the smoothed sample matrix $\mathbf{S}$. Since the columns of $\mathbf{S}$ are generated by applying relaxation to random test vectors, they are biased toward slow-to-converge error modes. The NLSS objective then selects a rank-$k$ subspace that captures the dominant components of these samples. This links the neural training objective to the coarse-correction mechanism used by multigrid.


\newpage

\end{document}